\documentclass[letterpaper]{article} 
\usepackage{aaai25}  
\usepackage{times}  
\usepackage{helvet}  
\usepackage{courier}  
\usepackage[hyphens]{url}  
\usepackage{graphicx} 
\usepackage{natbib}  
\usepackage{caption} 
\usepackage{algorithm}
\usepackage{algorithmic}

\usepackage{amsmath}
\usepackage{amsthm}
\usepackage{mathtools,amssymb}
\usepackage{bm}
\usepackage{subfig}
\usepackage{verbatimbox}
\usepackage{makecell}
\usepackage{booktabs}

\newtheorem{thm}{\bf Theorem}

\newtheorem{lemma}{\bf Lemma}

\newcommand{\ourmethod}{AdaZero}

\usepackage{newfloat}
\usepackage{listings}
\DeclareCaptionStyle{ruled}{labelfont=normalfont,labelsep=colon,strut=off} 
\floatstyle{ruled}
\newfloat{listing}{tb}{lst}{}
\floatname{listing}{Listing}
\title{The Exploration-Exploitation Dilemma Revisited:\\An Entropy Perspective}
 \author{
     Renye Yan\textsuperscript{\rm 1}\footnotemark[1], 
     Yaozhong Gan\textsuperscript{\rm 2}\thanks{Equal contribution.}, 
     You Wu\textsuperscript{\rm 3},  
     Ling Liang\textsuperscript{\rm 1}, 
     Junliang Xing\textsuperscript{\rm 4}, 
     Yimao Cai\textsuperscript{\rm 1}\thanks{Correspondence author.}, 
     Ru Huang\textsuperscript{\rm 1}
 }
 \affiliations{
     \textsuperscript{\rm 1}Peking University,
     \textsuperscript{\rm 2} Nanjing University of Aeronautics and Astronautics,
     \textsuperscript{\rm 3} Nanjing University,
     \textsuperscript{\rm 4} Tsinghua University

     victory@stu.pku.edu.cn, yzgancn@163.com, jlxing@tsinghua.edu.cn,
     \{lingliang, lingliang, caiyimao, ruhuang\}@pku.edu.cn
 }

\usepackage{bibentry}

\begin{document}
\maketitle

\begin{abstract}

The imbalance of exploration and exploitation has long been a significant challenge in reinforcement learning. In policy optimization, excessive reliance on exploration reduces learning efficiency, while over-dependence on exploitation might trap agents in local optima. This paper revisits the exploration-exploitation dilemma from the perspective of entropy by revealing the relationship between entropy and the dynamic adaptive process of exploration and exploitation. Based on this theoretical insight, we establish an end-to-end adaptive framework called \emph{\ourmethod{}}, which automatically determines whether to explore or to exploit as well as their balance of strength.
Experiments show that \ourmethod{} significantly outperforms baseline models across various Atari and MuJoCo environments with only a single setting. Especially in the challenging environment of Montezuma, \ourmethod{} boosts the final returns by up to fifteen times. Moreover, we conduct a series of visualization analyses to reveal the dynamics of our self-adaptive mechanism, demonstrating how entropy reflects and changes with respect to the agent's performance and adaptive process.

\end{abstract}

\section{Introduction}
In reinforcement learning (RL), the objective of policy optimization is to directly optimize a policy's parameters by maximizing the cumulative rewards~\cite{ahmed2019understanding,schulman2015trust,schulman2017proximal,bellemare2016unifying}. Due to the sparsity and delay of rewards~\cite{TNNLS23DRLSurvey, AAAI21Kang}, the policy optimization process is usually ineffective and hindered. 
To tackle this challenge, methods are proposed to improve exploration or exploitation respectively.
On the one hand, exploration-centric methods enhance the formation of new policies by expanding the search range of actions and/or states. For example, methods based on maximum entropy~\cite{ahmed2019understanding,haarnoja2018soft,zhang2024entropy,han2021max} encourages exploration by increasing randomness, and intrinsic reward methods~\cite{TAMD10IntrinsicMotivation, ICML19EMI, NIPS20IMExploration, IJCAI20IntrinsicExploration, NIPS22IELA, NIPS22E3B, ICML22MSEExploration} leverage state coverage to drive exploration. On the other hand, exploitation-centric methods~\cite{oh2018self,schulman2017proximal} execute the currently optimal policy based on the experience already learned by the agent, aiming to straightforwardly obtain the maximum cumulative returns.


While the exploration-centric and exploitation-centric methods both have made some progress, previous works pay less attention to the synergy of exploration and exploitation, which is crucial for policy optimization.
Balancing exploration and exploitation presents a complex dilemma in practice. Agents must consider multiple factors, including the environment's dynamism, local optima traps, and reward delays. Balancing short-term returns with long-term gains within limited time and resources is exceedingly challenging.
An ideal policy optimization process is expected to effectively switch between exploration and exploitation policies and adaptively balance the two effects.

A few attempts have been made to address the imbalance problem of exploration and exploitation, including decaying intrinsic reward after a certain number of steps~\cite{ICLR20NeverGiveUp,yan2023mnemonic}, and decoupling the training of exploration and exploitation policy~\cite{liu2021decoupling,colas2018gep}. Essentially, these methods are still exploration-centric, since the intrinsic rewards will never diminish to zero. In fact, as the extrinsic rewards are sparse, the non-zero intrinsic rewards, no longer how small, will dominate the training process at most of the time. Simply weakening the strength of exploration in a rigid manner will result in a sub-optimal policy.

In this paper, we revisit the exploration-exploitation dilemma from an entropy perspective. We start by theoretically revealing the relationship between entropy and intrinsic rewards in policy optimization, showing them changing synchronously in certain conditions.
From this viewpoint, we can treat the increase and decrease of entropy as proxies to indicate the agent's ability of exploration and exploitation in the current environment. We argue the presence and amount of intrinsic rewards should be determined by the agent's level of mastery about the environment. Hence, we formulate a dynamic adaptive process of policy optimization and derive several representative scenarios characterizing its self-adaptive capability of exploration and exploitation.


Based on these theoretical insights, we establish an end-to-end adaptive framework \emph{\ourmethod{}}. \ourmethod{} continues to evaluate the agent's mastery level in each state via an evaluation network and adaptively determines the balance of strength of exploration and exploitation.
The state autoencoder receives real-time state images and outputs the reconstruction error as intrinsic rewards to encourage exploration. Meanwhile, the reconstructed state image generated by the autoencoder, which represents the agent's mastery level of the current state, is provided to an evaluation network to determine the strength of exploration and exploitation.
Guided by this mechanism, agents can automatically adapt to environmental changes and decide whether to explore or to exploit, as well as the strength of both, thus breaking out from the dilemma of imbalanced exploration-exploitation.
The adaptive mechanism of \ourmethod{} is entirely realized in an end-to-end manner without any need to manually design a decaying scheme.

To provide empirical support for our theoretical findings and to validate the effectiveness of our proposed self-adaptive mechanism, we experiment with \ourmethod{} on 63 Atari and
MuJoCo tasks with continuous and discrete action state spaces. The results show that without any customized reward design or tuning of hyperparameters for each environment, \ourmethod{} outperforms previous methods by a large margin across simple to difficult environments, demonstrating the superiority of our self-adaptive mechanism. We also perform a series of visualization analyses to reveal the functionality of entropy in policy optimization and how \ourmethod{}'s automatic adaptive exploration and exploitation effectively works.

\begin{figure*}
    \centering
    \includegraphics[width=0.9\linewidth]{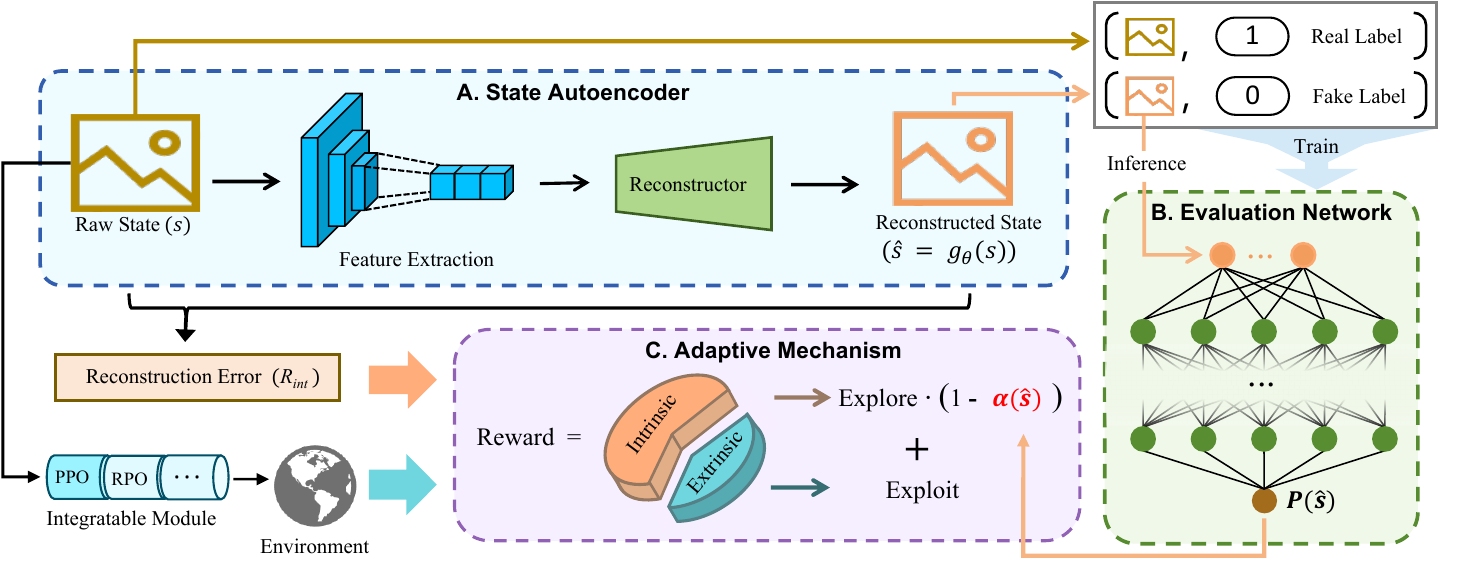}
    \caption{\textbf{\ourmethod{}'s Framework}. \ourmethod{} consists of three main components: (A) State Autoencoder, (B) Evaluation Network for level of mastery, and (C) Adaptive Mechanism. The state autoencoder encodes and reconstructs states in raw images, where the reconstruction errors work as the driving force for the agent's exploration. The mastery evaluation network evaluates the reconstructed states and outputs the probability of $\hat{s}$ being real images as the balance factor $\alpha(\hat{s})$. Finally, $\alpha(\hat{s})$ is used in the adaptive mechanism to dynamically balance exploration and exploitation.}
    \label{framework}
\end{figure*}

\section{Preliminaries}
\paragraph{Markov Decision Process (MDP).}
An MDP~\cite{sutton1999reinforcement,puterman2014markov} is described by the tuple $\left\langle \mathcal{S}, \mathcal{A}, P, R, \gamma \right\rangle$, where $\mathcal{S}$ and $\mathcal{A}$ represent state and action spaces respectively. The function $P$ is the transition probability function, indicating the likelihood of moving from one state to another following an action. The function $R$ is the reward function and is supposed to be non-negative. The discount factor $\gamma \in [0, 1)$ modifies the value of future rewards. The agent's interactions within the MDP generate a trajectory $\tau$ of states, actions, and rewards: $s_0, a_0, R(s_0, a_0), \cdots, s_t, a_t, R(s_t, a_t), \cdots$. The state-action value function $Q^{\pi}(s_t, a_t)$ is defined as follows:
\begin{equation}
\label{q}
    Q^{\pi}(s_t, a_t) = \mathbb{E}_{\tau \sim \pi}\left[\sum_{i=0}^{\infty} \gamma^i R(s_{t+i}, a_{t+i}) \mid s_t, a_t\right]. 
\end{equation}

The policy is defined in the softmax form of the Q value:
\begin{equation}
\label{policy}
    \pi(a_i|s) = \frac{e^{Q(s, a_i)}}{\sum_{k}e^{Q(s, a_k)}}, i= 1, 2, \cdots, n.
\end{equation}

In order to guide the agent to make decisions that maximize long-term gains, the goal of policy optimization is to learn a policy $\pi$ that maximizes the expected total discounted reward $\eta(\pi)$, defined as follows:
\[ \eta(\pi) = \mathbb{E}_{\tau \sim \pi}\left[\sum_{i=0}^{\infty} \gamma^i R(s_{i}, a_{i})\right]. \]


\paragraph{Intrinsic-Based RL.}
Extending the traditional MDP framework, intrinsic rewards are introduced to help agents explore in sparse rewards tasks. The total reward function $R_{\text{total}}(s, a)$ combines extrinsic rewards $R_{\text{ext}}$ and intrinsic rewards $R_{\text{int}}$ and is defined as follows:
\begin{equation*}
    R_{\text{total}}(s, a)= R_{\text{ext}}(s, a) + R_{\text{int}}(s, a).
\end{equation*}

The Q-function is updated with intrinsic rewards added:
\begin{equation}\label{qtotal}
\begin{aligned}
Q_{\text{total}}(s,a)&=E_{\tau \sim \pi}\left[\sum_{i}\gamma^i (R_{\text{ext}}(s_i,a_i)+R_{\text{int}}(s_i,a_i))|s,a\right]\\&\triangleq Q_{\text{ext}}(s,a)+\delta(s, a),
\end{aligned}
\end{equation}
where $Q_{\text{ext}}(s, a) = E_{\tau}\left[\sum_{i}\gamma^i R_{\text{ext}}(s_i,a_i)|s,a\right]$, and $\delta(s, a) = E_{\tau}\left[\sum_{i}\gamma^i R_{\text{int}}(s_i,a_i)|s,a\right]$ is the accumulated return with only intrinsic rewards.


\section{Method}

An ideal optimization process should cyclically alternate between exploration and exploitation according to actual development conditions.
Particularly, reasonable exploitation can help the agent to enter new states and hence bring more in-depth exploration. 
Motivated by the above observations, we revisit the classic problem - the exploration-exploitation dilemma in RL - from the perspective of entropy.
In Subsection~\ref{adapt}, we investigate the relationship of two prominent factors associated to exploration, i.e., entropy and intrinsic rewards.
We theoretically reveal that entropy and intrinsic rewards change synchronously under mild conditions, and hence extend the Bellman equation to formulate a self-adaptive mechanism. 
In Subsection~\ref{ourframework}, we instantiate our adaptive framework \ourmethod{}, which is entirely based on end-to-end networks.

\label{method}
\subsection{Exploration-Exploitation Adaptation from
the Perspective of Entropy}
\label{adapt}

This subsection investigates the relationship between entropy and intrinsic rewards.
Without loss of generality, we consider the action space of instrinsic-based RL to consist of two actions, $a_1$ and $a_2$, and suppose $a_1$ is the optimal action and $a_2$ is the suboptimal action, i.e., $\delta(s, a_1) \leq \delta(s, a_2)$.

According to Eqn.~(\ref{policy}), we can define the policies $\pi_{\text{ext}}$ and $\pi_{\text{total}}$, which are the softmax over $Q_{\text{ext}}$ and $Q_{\text{total}}$. Next, we first show that entropy and intrinsic rewards change synchronously in mild conditions. Then, we derive an adaptive mechanism of exploration and exploitation.

\begin{lemma}
\label{thm1}
    If $0 \leq \delta(s, a_2) - \delta(s, a_1) \leq 2(Q_{\text{ext}}(s, a_1) - Q_{\text{ext}}(s, a_2))$ holds, we have $$ H(\pi_{\text{ext}}|s) \leq H(\pi_{\text{total}}|s). $$
\end{lemma}
\begin{proof}
    By definition, we have

    $ \pi_{\text{ext}}(a_i|s) = \frac{e^{Q_{\text{ext}}(s, a_i)}}{e^{Q_{\text{ext}}(s, a_1)} + e^{Q_{\text{ext}}(s, a_2)}} $ and $ \pi_{\text{total}}(a_i|s) = \frac{e^{Q_{\text{ext}}(s, a_i) + \delta(s, a_i)}}{e^{Q_{\text{ext}}(s, a_1) + \delta(s, a_1)} + e^{Q_{\text{ext}}(s, a_2) + \delta(s, a_2)}} $.
    
    For the optimal action $a_1$, we have
    \begin{align*}
        \pi_{\text{total}}(a_1|s) = &\frac{e^{Q_{\text{ext}}(s, a_1) + \delta(s, a_1)}}{e^{Q_{\text{ext}}(s, a_1) + \delta(s, a_1)} + e^{Q_{\text{ext}}(s, a_2) + \delta(s, a_2)}}\\
        = &\frac{e^{Q_{\text{ext}}(s, a_1)}}{e^{Q_{\text{ext}}(s, a_1)} + e^{Q_{\text{ext}}(s, a_2)} \cdot e^{\delta(s, a_2) - \delta(s, a_1)}}\\
        \leq &\pi_{\text{ext}}(a_1|s).
    \end{align*}

    For the suboptimal action $a_2$, we have $ \pi_{\text{total}}(a_2|s) \geq \pi_{\text{ext}}(a_2|s) $ according to $ \pi(a_1|s) + \pi(a_2|s) = 1 $.
    
    Under $ 0 \leq \delta(s, a_2) - \delta(s, a_1) \leq 2[Q_{\text{ext}}(s, a_1) - Q_{\text{ext}}(s, a_2)] $,
    we get $ \pi_{\text{total}}(a_2|s) \leq \pi_{\text{ext}}(a_1|s) $.

    According to the monotonicity of entropy (Appendix~\ref{monotonicity}), the entropy of the policy considering two actions increases monotonically with respect to $\pi(a_i|s)$ in the interval (0, 0.5) and decreases monotonically in the interval (0.5, 1).
    Combining the properties of entropy with respect to policy probabilities, we obtain $ H(\pi_{\text{ext}}|s) \leq H(\pi_{\text{total}}|s) $.
\end{proof}

\paragraph{Self-Adaptive Bellman Equation.}
From lemma~\ref{thm1}, we establish a connection between entropy and intrinsic rewards.
Next, we extend the Bellman equation to formulate a new dynamic adaptive mechanism for exploration and exploitation. We define the self-adaptive equation as follows:
\begin{equation}
\label{eqn6}
    \begin{aligned}
        &Q_{\text{total}}(s,a)\\
=&E_{\tau}\left[\sum_{i}\gamma^i (R_{\text{ext}}(s_i,a_i)+ (1-\alpha(s_i))R_{\text{int}}(s_i,a_i))|s,a\right]\\
\triangleq &Q_{\text{ext}}(s,a)+ \hat{\delta}(s, a),
    \end{aligned}
\end{equation}

where $Q_{\text{ext}}(s, a)$ is the same as in Eqn. (\ref{qtotal}), $\hat{\delta}(s, a) = E_{\tau}\left[\sum_{i}\gamma^i (1-\alpha(s_i))R_{\text{int}}(s_i,a_i)|s,a\right]$, and $\alpha(s_i) \in [0,1]$ is a function of state $s_i$ indicating the mastery level of $s_i$.

Based on Eqn.~(\ref{eqn6}), we can derive three typical scenarios characterizing the exploration-exploitation adaptive mechanism from the perspective of entropy, which also provides a theoretical explanation for how entropy and intrinsic rewards expedite exploration.

\begin{thm}
\label{thm2}
    Under conditions of lemma~\ref{thm1} and different values of $\alpha$ function,
    for any $s$, we have $$ H(\pi_{\text{ext}}|s) \leq H(\pi_{\text{total}}|s) \text{ or } H(\pi_{\text{ext}}|s) > H(\pi_{\text{total}}|s) .$$
\end{thm}
\begin{figure*}
    \centering
    \includegraphics[width=0.99\linewidth]{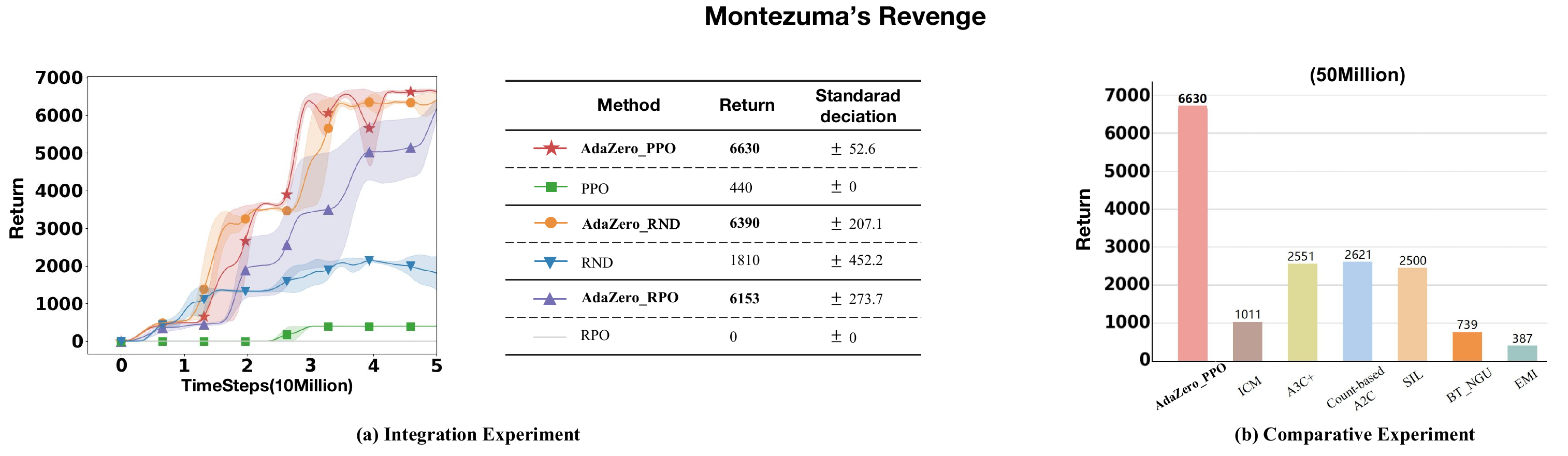}
    \caption{Main experiments in the most challenging \textbf{Montezuma's Revenge}. (a) We integrate \ourmethod{} into three representative RL methods and show that \ourmethod{} can bring significant improvements; (b) Our method also presents an advantageous performance compared with other advanced baselines. \textbf{Equipped with \ourmethod{}, RND reached the score published in the original paper within only one-tenth of the training steps}.}
    \label{main_result}
\end{figure*}

\begin{proof}
    According to Eqn.~(\ref{eqn6}), we have
    \begin{equation*}
    \hat{\delta}(s, a) = (1-\alpha(s))R_{\text{int}}(s,a)+\gamma \mathbb{E}_{s',a'|s,a}[\hat{\delta}(s', a')].
    \end{equation*}
    Depending on the mastery level of the states $\alpha(s)$, there will be different scenarios regarding the relationship between exploration and exploitation. We mainly discuss the following three typical cases:

    \emph{Case I - Exploration Dominant}: for any $s$, when $\alpha(s) \equiv 0$, $Q_{\text{total}}(s, a_1) = Q_{\text{ext}}(s, a_1) + \delta(s, a_1)$, and $Q_{\text{total}}(s, a_2) = Q_{\text{ext}}(s, a_2) + \delta(s, a_2)$. From lemma~\ref{thm1}, we know $H(\pi_{\text{ext}}|s) < H(\pi_{\text{total}}|s)$, i.e., the entropy increases. In this scenario, the strength of exploration reaches the peak level, which means our method exhibits stronger exploratory capabilities than methods without intrinsic rewards.

    \emph{Case II - Adaptive Exploration-Exploitation}: $\alpha(s)$ satisfies $Q_{\text{total}}(s, a_1) = Q_{\text{ext}}(s, a_1) + \hat{\delta}(s, a_1)$ and $Q_{\text{total}}(s, a_2) = Q_{\text{ext}}(s, a_2)$. This means $H(\pi_{\text{ext}}|s) > H(\pi_{\text{total}}|s)$, i.e., the probability of optimal policy increases while entropy starts to decrease. In this scenario, the degree of exploitation adaptively increases while the level of exploration decreases. 

    \emph{Case III - Exploitation Dominant}: for any $s$, when $\alpha(s) \equiv 1$, we have $Q_{\text{total}}(s, a_1) = Q_{ext}(s, a_1)$, and $Q_{\text{total}}(s, a_2) = Q_{\text{ext}}(s, a_2)$. This means the entropy remains unchanged, i.e., $H(\pi_{\text{ext}}|s) = H(\pi_{\text{total}}|s)$. In this scenario, the strength of exploration is zeroed out and exploitation is dominant.
\end{proof}

From the above theorem, we can see that as $\alpha(s)$ changes, the entropy of our algorithm can get higher or lower than that of $\pi_{\text{ext}}$. This implies that our method can conduct a certain level of exploration while leveraging existing experience, compared with traditional RL methods like PPO~\cite{schulman2017proximal,mnih2015human}. Besides, in contrast to existing exploration-centric methods based on intrinsic rewards or maximum entropy, our method can alternate between exploration and exploitation adaptively.
This theoretical advantage of our method will be further confirmed in our visualization experiment (Figure~\ref{see} and Section~\ref{vis}).

\subsection{\ourmethod{}’s Framework}
\label{ourframework}
Inspired by the above adaptive mechanisms, we formalize an end-to-end exploration-exploitation dynamic adaptive framework — \ourmethod{}. As illustrated in Figure~\ref{framework}, the \ourmethod{} framework consists of three parts: a state autoencoder, a mastery evaluation network and the policy optimization mechanism (referred to as the adaptive mechanism).

\paragraph{State Autoencoder.} The state autoencoder aims to motivate the agent's exploratory behavior of searching for unknown policies to get familiar with the environment. As shown in Figure~\ref{framework} (A), the state autoencoder receives raw states $s$ via the interaction between the agent and the environment and learns to reconstruct the input images $\hat{s}=g_\theta(s)$, where $\theta$ is network parameters. The reconstruction error $R_{\text{int}}(s,a)$ is defined as follows: 
\begin{equation*}
R_{\text{int}}(s,a) =\mathcal{L}_g(\theta) = \frac{1}{2} \left\| s -  g_\theta(s) \right\|_2^2 .
\end{equation*}

We use the reconstruction error as intrinsic rewards to drive exploration.
A larger error indicates insufficient training of the network on the state, and hence enhanced exploration is needed for this state. Oppositely, a smaller error indicates that the training on that state has been sufficient, leading to an automatic decrease of intrinsic reward.
Note that we only need a single network to estimate intrinsic rewards in contrast to competitive methods like RND and NGU. Moreover, our state autoencoder operates on raw images rather than latent representations of states, leading to a more accurate measurement of state novelty.

\paragraph{Mastery Evaluation Network.}
\begin{figure*}
    \centering
    \includegraphics[width=0.81\linewidth]{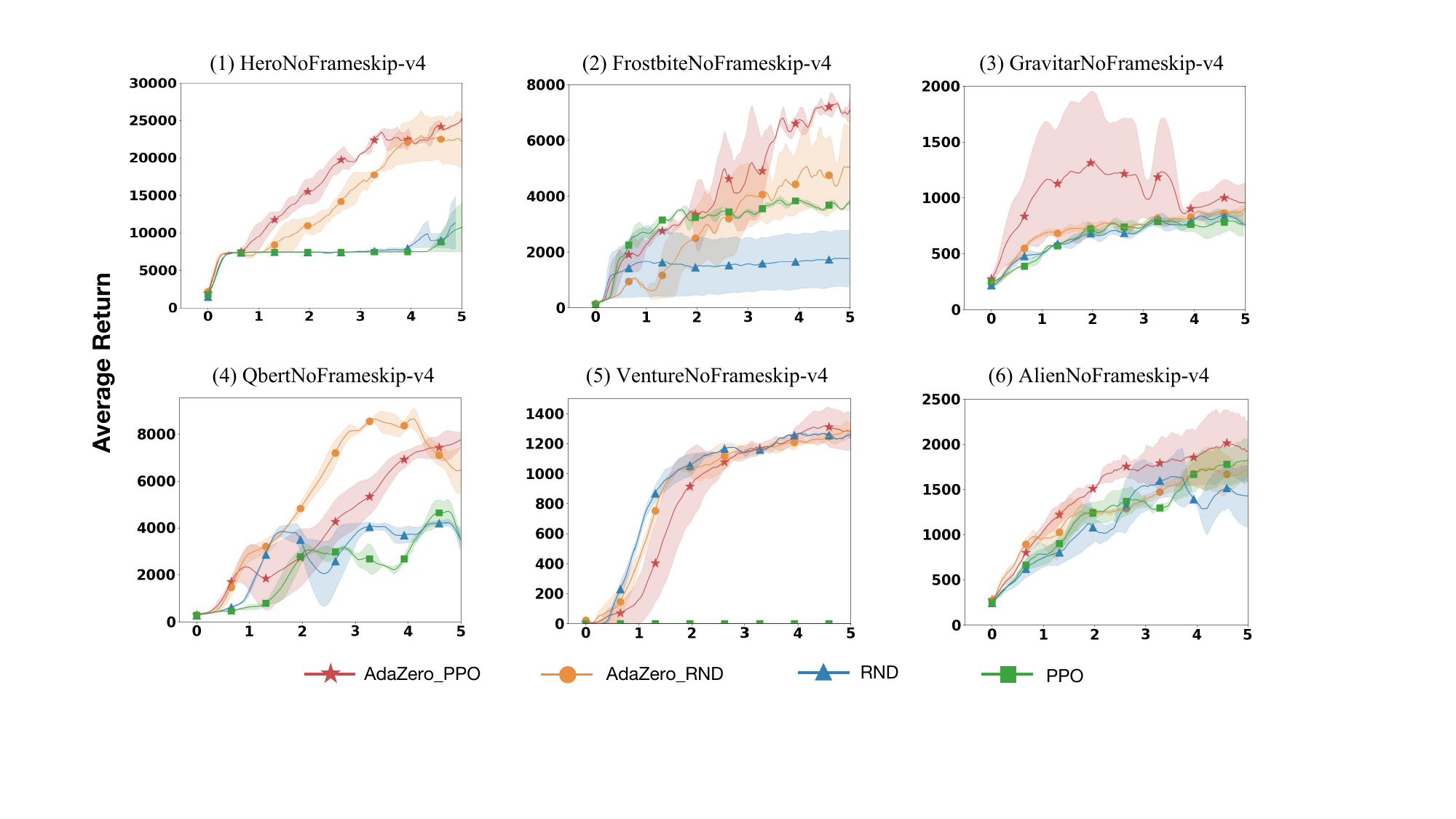}
    \caption{\textbf{Generalization experiments in discrete-space environments (Atari).} The x-axis represents timesteps in 10 million.
    }
    \label{atari_hard}
\end{figure*}

While conducting effective exploration is essential in environments with sparse rewards and delayed feedback, solely focusing on exploration without exploitation can potentially reducing learning efficiency as the agent may spend excessive time exploring ultimately ineffective policies. Therefore, correctly determining where and how much to explore based on the agent's mastery level of the states is crucial for the agent's performance. To this end, \ourmethod{} involves a mastery evaluation network to provide real-time assessment of the necessity of exploration, as shown in Figure~\ref{framework} (B).

The mastery network is a binary classifier denoted as $f_{\omega}$. At the training stage, the network receives mixed data of real state images and fake state images reconstructed by the state autoencoder, denoted as $s$ and $\hat{s}$ respectively. It is trained to distinguish between the two kinds of inputs and output the probability of the inputs being real images, with a cross-entropy loss defined as follows:
\begin{equation*}
\mathcal{L}_f(\omega)=\mathbb{E}_{s \sim \mathcal{S}}[\log (1-f_{\omega}(s))]+\mathbb{E}_{\hat{s} \sim \hat{\mathcal{S}}}\left[\log \left(f_{\omega}(\hat{s})\right)\right],
\end{equation*}
where $\omega$ is the network's parameters, $\mathcal{S}$ and $\hat{\mathcal{S}}$ are the set of real state images and reconstructed images respectively.

At the inference stage, the evaluation network only receives the reconstructed images from the state autoencoder and outputs the probability $\alpha(\hat{s}) \triangleq f_{\omega}(\hat{s}) $. Hence, $\alpha(\hat{s}) \in [0, 1]$ can naturally serve as a measure of \ourmethod{}'s grasp of the current training state information.

\paragraph{Adaptive Mechanisms.}
Inspired by Eqn.~\ref{eqn6} and Theorem~\ref{thm2}, By integrating the reconstruction error $R_{\text{int}}$ provided by the state autoencoder and the estimated level of mastery of the state information $\alpha(\hat{s})$ from the evaluation network, we implement the adaptive mechanism of \ourmethod{} depicted in Figure~\ref{framework} (C) through the following formula:
\begin{equation}
\label{int}
  R_{\text{total}}(s,a) = R_{\text{ext}}(s,a) +  (1-\alpha(\hat{s}))R_{\text{int}}(s,a) .
\end{equation}

When the agent reaches a high level of mastery in current state, the exploration intensity should be reduced to increase the exploitation ratio. Conversely, when the level of mastery is low, the exploration intensity should be increased to encourage the agent to better learn from the various information about the state. As $\alpha(\hat{s}) \in [0,1]$ reflects the agent's ability to grasp valuable information in the reconstructed states, we leverage $(1-\alpha(\hat{s}))$ to adaptively control the strength of intrinsic rewards, and hence dynamically balance the exploration and exploitation.
Note that compared with methods that simply decays intrinsic rewards~\cite{ICLR20NeverGiveUp,yan2023mnemonic}, the adaptive mechanism of \ourmethod{} is entirely realized by end-to-end networks without any need to manually design a decaying scheme. It can adaptively increase or decrease intrinsic rewards according to the training progress and environment changes.

\section{Evaluation Experiments}
\label{exp}


We conduct experiments on 65 different RL tasks, range from dense to sparse reward environments, from continuous to discrete action spaces, and from performance comparison to visualization analyses, which validate the applicability and effectiveness of our \ourmethod{} framework and provide empirical support our theoretical analysis.

We choose to compare \ourmethod{} with the most effective and commonly-used RL algorithms, including traditional PPO~\cite{schulman2017proximal}, competitive exploration-centric RND~\cite{burda2018exploration}, and the recent exploitation-centric RPO~\cite{reflectivepo}.
As \ourmethod{} is designed to be flexible and can be easily integrated into existing algorithms, 
we first compare the performance of the baselines with versus without \ourmethod{} integrated.
Then we focus on PPO with \ourmethod{} as a representative of our method. In the remainder of this paper, \ourmethod{} refers to PPO integrated with \ourmethod{} unless otherwise specified.

For each environment, we run three times with different random seeds and each containing 50 million steps as commonly done in this field. We draw the curves of average return across steps with the solid lines where the shading on each curve represents the variance across three runs.
\begin{figure}
    \centering
    \includegraphics[width=0.99\linewidth]{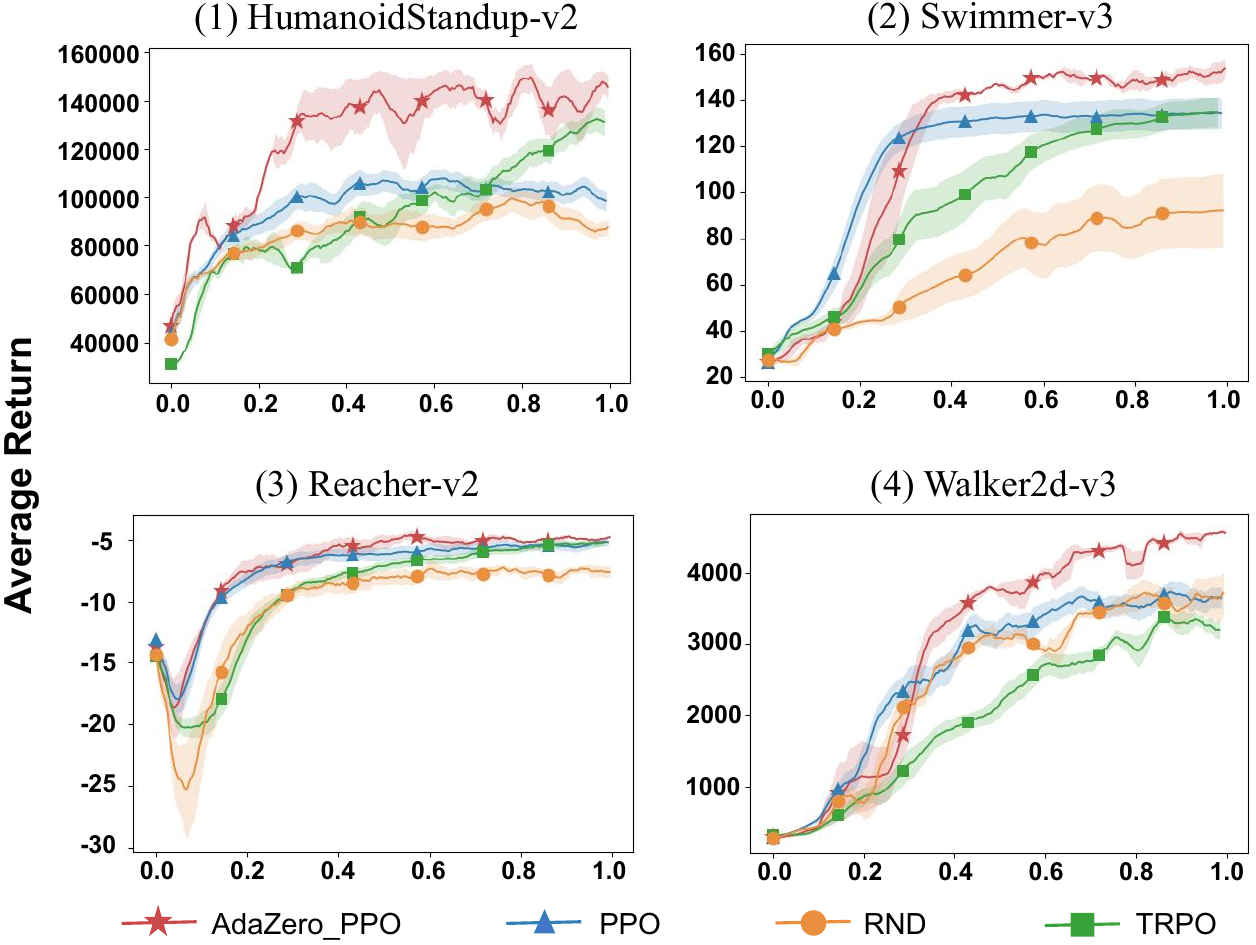}
    \caption{\textbf{Generalization experiments in MuJuCo}, showing the generalizable performance of \ourmethod{} in different types of \textbf{continuous space tasks}. The x-axis represents timesteps in million.
    }
    \label{mujuco}
\end{figure}

\begin{figure}
    \centering
    \includegraphics[width=0.99\linewidth]{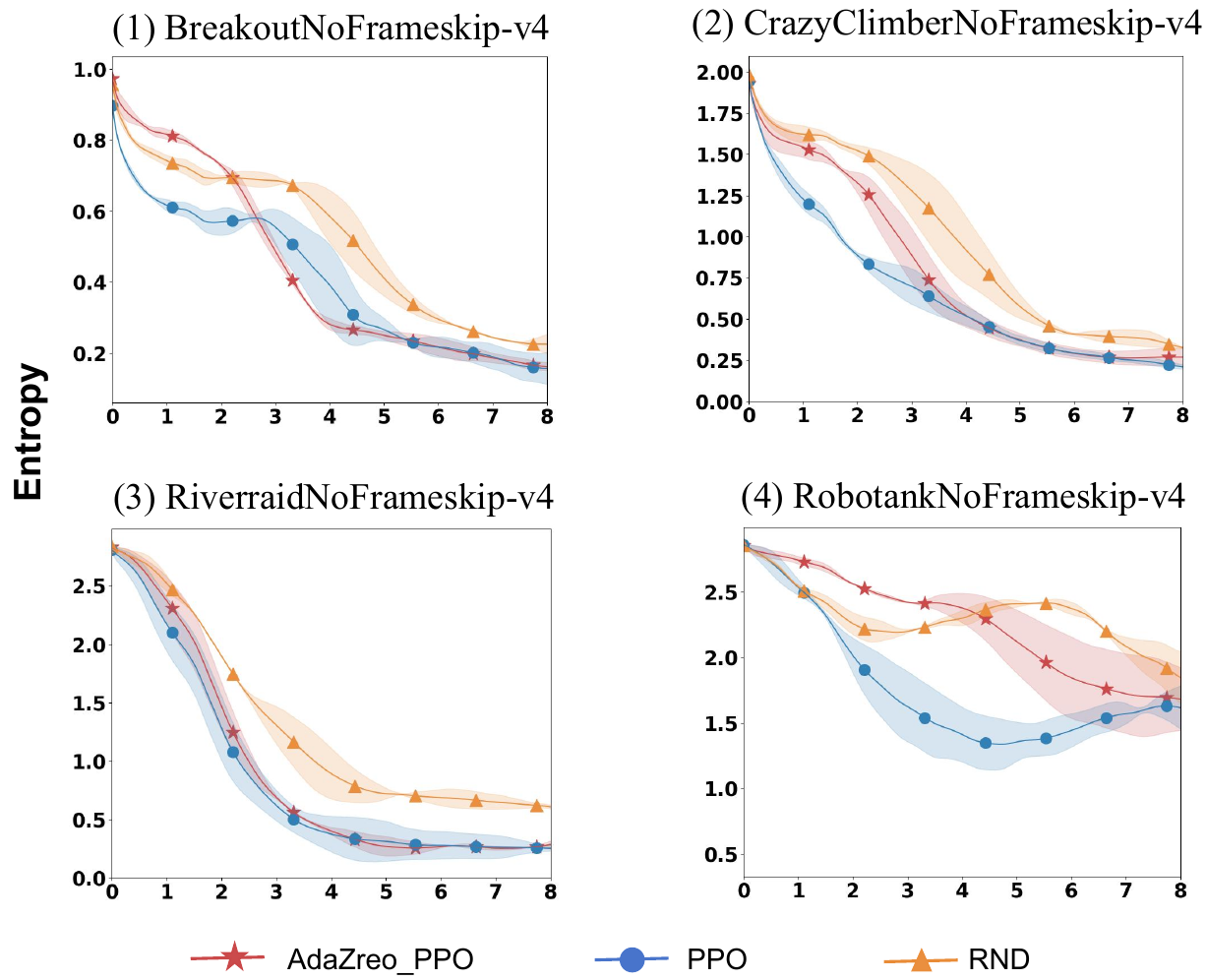}
    \caption{\textbf{Entropy Visualization} of \ourmethod{} vs. PPO and RND in Atari. The x-axis represents timesteps in million.
    }
    \label{see}
\end{figure}

\subsection{Main Experiments}

We conducted integration experiments in the most challenging RL environment, Montezuma's Revenge. The results are shown in Figure~\ref{main_result}. Overall, \ourmethod{} boosts the total returns up to about fifteen times for PPO and 3.5 times for RND (Figure~\ref{main_result} (a)). Besides, \ourmethod{} achieves more than twice returns compared with other advanced baselines (Figure~\ref{main_result} (b)). The results demonstrate that within the same training steps, \ourmethod{} significantly surpasses other algorithms in terms of final returns.

The integration experiment results in Figure~\ref{main_result} (a) also confirm \ourmethod{}'s effective integratability and superior adaptability in highly challenging environments.
Compared with algorithms like PPO and RPO, which primarily focus on exploitation, \ourmethod{} maintains a high exploitation rate and also supports selective exploration. Compared with the exploration-centric RND, \ourmethod{} enhances its exploitation ability. Moreover, it can be observed that there are more horizontal fluctuations in curves of the methods integrated with \ourmethod{}. This indicates that our dynamic adaptive mechanism enables the agent to discover more localized policies during training, thereby enhancing the overall training efficiency. Without \ourmethod{}, PPO and RPO would be trapped in invalid policies and almost fail.

\subsection{Generalization Experiments}
To demonstrate \ourmethod{}'s broad applicability, this subsection presents its performance in a number of well-known challenging RL environments, including all the other 58 discrete-space Atari games as well as the continuous-space environments of MuJoCo.

Figure~\ref{atari_hard} demonstrates the performance of \ourmethod{} in six well-known challenging Atari environments apart from Montezuma. \ourmethod{} significantly improves performance in almost all environments, except for Venture where the improvement is less pronounced. The improvement is most notable in Hero, where the performance is approximately four times better than the baseline. The results indicates that \ourmethod{} is effective in challenging discrete space tasks. The additional results for the remaining 52 Atari environments is included in Appendix~\ref{supplementary_experiment}.

Figure~\ref{mujuco} shows the performance of \ourmethod{} in continuous action spaces. The results indicate that our algorithm outperforms advanced algorithms in the HumanoidStandup, Swimmer, Reacher, and Walker environments. These results demonstrate that \ourmethod{} is also broadly effective in continuous spaces.

\section{Visualization Analysis}

To provide empirical support for our theoretical findings and to validate the effectiveness of our proposed self-adaptive mechanism, in this section, we present a series of visualization analyses to reveal the influence of entropy in policy optimization and how \ourmethod{}'s automatic adaptive exploration and exploitation effectively works.

\begin{figure}
    \centering
    \includegraphics[width=0.9\linewidth]{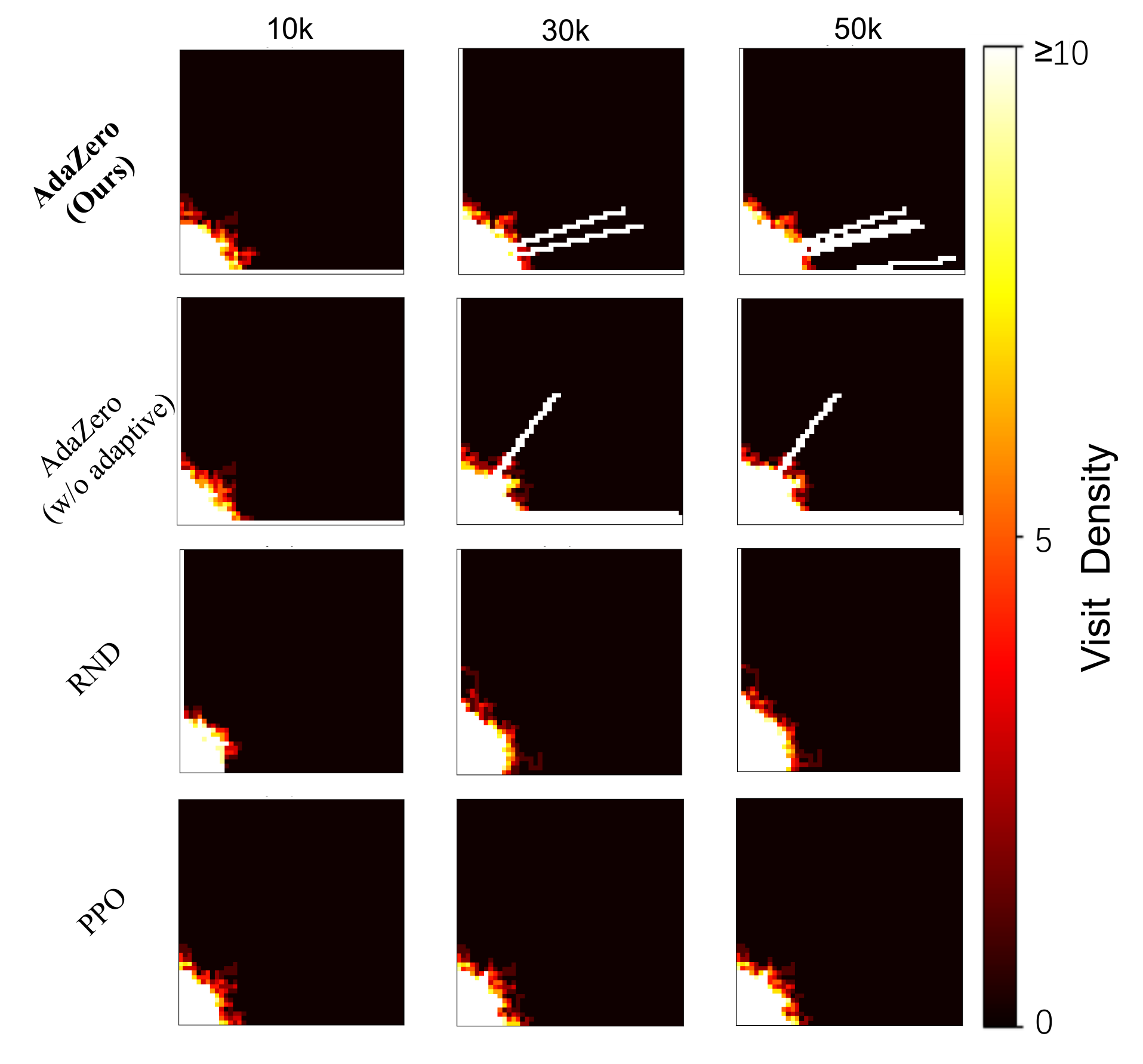}
    \caption{\textbf{Exploration Visualization in Dark Chamber} in 50k steps. The exploration radius of \ourmethod{} with adaptive mechanism is the largest, leading to the largest number of valid paths.
    }
    \label{darkchamber}
\end{figure}

\subsection{Dynamic Adaptive Visualization}
\label{vis}
To more clearly demonstrate the entropy changes in \ourmethod{} during the training process and to provide emperical evidence for the proof in subsection \ref{adapt}, we visualize the entropy curves of \ourmethod{} and baseline algorithms in four commonly used Atari games.

Based on the visualization experiment in Figure \ref{see}, we observed that the entropy of \ourmethod{} is lower than that of the representative exploration method RND. Compared to the traditional exploitation-centric PPO, the entropy of \ourmethod{} during training may be higher or lower than that of PPO.
When \ourmethod{}'s entropy lies between PPO and RND, it is in a dynamic adaptive phase of exploration and exploitation.

It is particularly noteworthy that the entropy of \ourmethod{} is even lower than that of the pure exploitation PPO algorithm in many cases. This is because \ourmethod{} can find more certain and better policies through the support of adaptive mechanisms, resulting in extremely low entropy.

\begin{figure}
    \centering
    \includegraphics[width=0.90\linewidth]{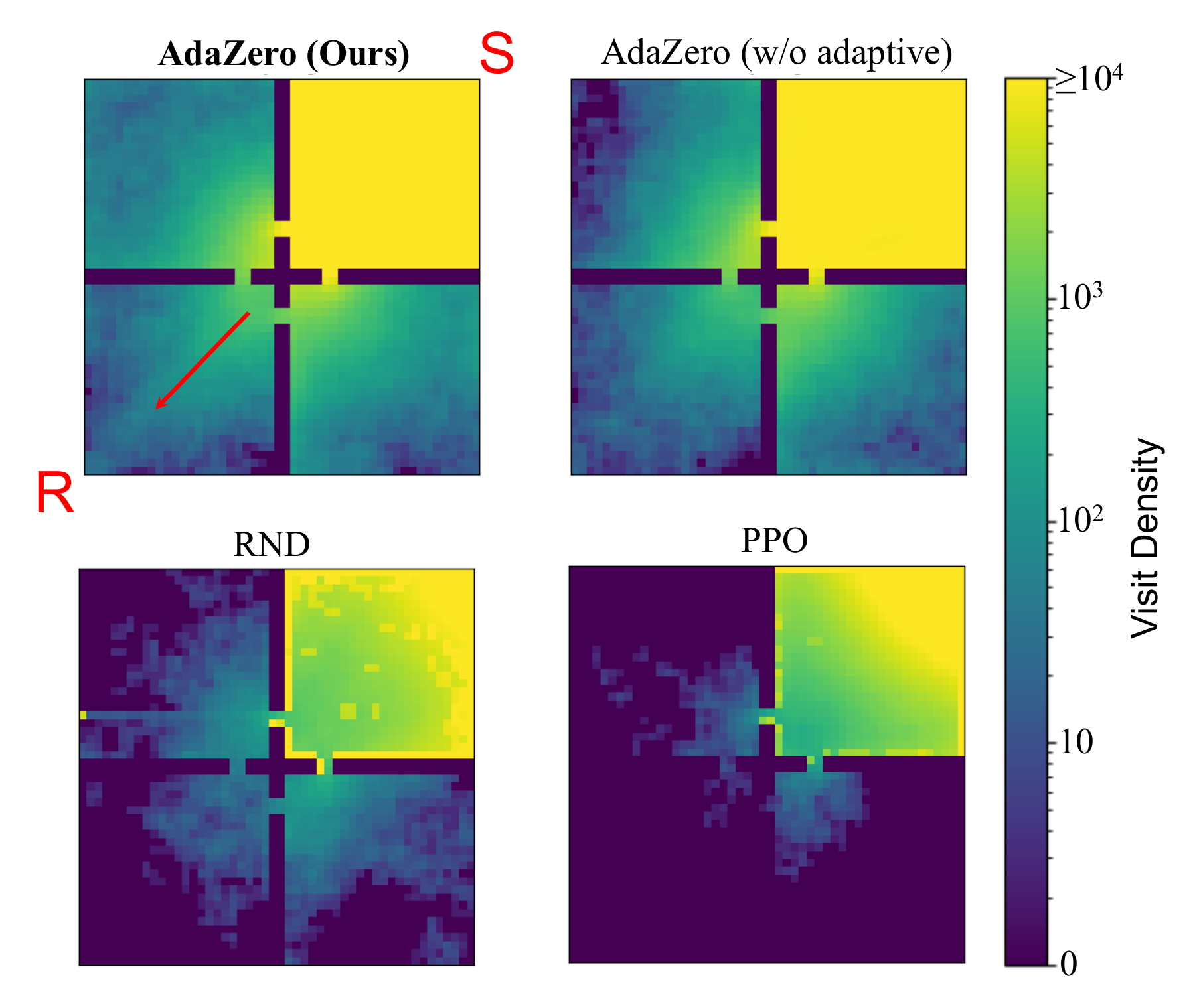}
    \caption{\textbf{Visualization Experiments in Four Rooms} in 2.5 million steps. \ourmethod{} can form an efficient strategy heading for the reward. 
    }
    \label{fourrooms}
\end{figure}

\subsection{Entropy Variation on Policy Optimization}
As shown in Figure~\ref{alignment} in the Appendix, we visualize the entropy changes along with the changes of performance and intrinsic rewards. We also present how the state mastery indicator $\alpha(\hat{s})$ changes during the policy optimization process.
The highlighted area (ii) in Figure~\ref{alignment} (a) reveals that as entropy stabilizes and decreases, performance gradually improves as the agent shifts from exploration to exploitation and discovers and reinforces high-reward strategies. Notably, the highlighted area (i) of Figure~\ref{alignment} (a) indicates that when entropy increases the return decreases, supporting the argument that excessive exploration at the expense of exploitation can ultimately hinder exploration itself. This finding is further confirmed by the visualization in Figure~\ref{darkchamber}, where the exploration areas of AdaZero and its ablation version without adaptive, are compared in the Four Rooms environment.

In addition, Figure~\ref{alignment} (b) shows the consistent trend between entropy and intrinsic rewards, providing experimental evidence for Lemma~\ref{thm1} in this paper.
Figure~\ref{alignment} (c) presents the changes in the state mastery indicator $\alpha(\hat{s})$ during training, where it can be observed that when $\alpha(\hat{s})$ reaches 1 in several instances, AdaZero has fully mastered the current state information, eliminating the need for further exploration, and intrinsic rewards vanish as the strategy shifts to complete exploitation.

\subsection{Effectiveness of Exploration}
To more clearly demonstrate AdaZero's exploration capability, we first constructed a \emph{Dark Chamber} environment, as shown in Figure~\ref{darkchamber}. The environment is 50x50 in size without any reward or penalty, with the agent starting its exploration from the bottom-left corner. We provide a state visit density map to illustrate the agent's exploration ability in the absence of rewards.

Moreover, to evaluate the agent's deep exploration capability in a sparse-reward environment, we conducted a visualization experiment in the \emph{Four Rooms} environment (Figure~\ref{fourrooms0} in Appendix). The agent starts at point s in the top-right corner, explores the environment to avoid wall obstacles, and eventually reaches the point in the bottom-left corner to obtain a reward. As shown in Figure~\ref{fourrooms}, AdaZero achieves the largest exploration area, nearly completing a full exploration of the environment and establishing a stable shortest-path strategy. The direction indicated by the arrows shows that our method forms the shortest path, demonstrating AdaZero's superior exploration ability. In contrast, the ablation version, AdaZero without adaptive, though performing better than the baseline, has a smaller exploration area and a less prominent shortest path than AdaZero, highlighting the importance of AdaZero's adaptive mechanism.
\section{Related Work}


Sparse reward environments have long posed a significant challenge in RL. Researchers have explored various methods to address this issue, with intrinsic rewards\cite{ICLR19ERND} and maximum entropy approaches\cite{ICML22MSEExploration} being among the most effective. The former encourages agents to discover novel states by measuring novelty through state visit density, while the latter enhances exploration quality by promoting randomness in policies. Despite progress, these methods still struggle with adaptive coordination between exploration and exploitation.

A few studies have attempted to address this challenge. For instance, intrinsic reward methods\cite{ICLR20NeverGiveUp} with decay factors introduce an annealing-like mechanism into the reward function, and approaches like decoupled reinforcement learning\cite{schafer2021decoupling} separately optimize exploration and exploitation strategies. However, these methods fail to achieve a truly adaptive balance between exploration and exploitation. The primary limitation is that the intrinsic motivation terms cannot fully adaptively decay to zero based on the environment and training progress; instead, they gradually reduce to zero. Moreover, these methods rely on manually designed decay coefficients or strategy separation and merging based on human experience.
For a detailed discussion of related literature, refer to Appendix A for additional related works.

\section{Conclusion}
This paper demonstrates the role of entropy in policy optimization. Based on this motivation, we re-examined the exploration-exploitation dilemma from the perspective of entropy and proposed an adaptive exploration-exploitation mechanism. Inspired by this mechanism, we formalized the \ourmethod{} which achieves adaptation based on end-to-end networks. and conducted extensive experiments in up to 65 environments. Both theoretical analysis and experimental results demonstrate that, compared to traditional RL algorithms, our method enables selective exploration while exploiting. Compared to exploration-dominant algorithms, our method allows for selective exploitation while exploring.


\bibliography{aaai25}
\clearpage
\appendix

\section{Monotonicity of Entropy}
\label{monotonicity}
The entropy of the policy $\pi$ is defined as:
\[ H(\pi|s) = -\sum_{a \in \mathcal{A}} \pi(a|s) \log \pi(a|s). \]

For the scenario of two actions, consider the following constrained optimization problem:

\begin{equation}
\begin{array}{ll} 
& \max H(\pi|s) \\
\text { s.t. } & \pi(a_1|s) + \pi(a_2|s) = 1.
\end{array}
\end{equation}

Using the Lagrange function method, we can define $ L(\pi) = -(\pi(a_1|s) \log \pi(a_1|s) + \pi(a_2|s) \log \pi(a_2|s)) + \beta (\pi(a_1|s) + \pi(a_2|s) - 1) $, where $\beta$ is the coefficient.
Taking the partial derivatives of the above function with respect to the policy, we have
\begin{equation}\label{piandao}
\frac{\partial L(\pi)}{\partial \pi(a_i|s)} = -(\log \pi(a_i|s) + 1) + \beta, \quad i=1,2.
\end{equation}
Setting the partial derivative to zero, we get $ \pi(a_1|s) = \pi(a_2|s) = e^{\beta - 1} = \frac{1}{2} $. 
We can see that the entropy increases monotonically with respect to $\pi(a_i|s)$ in the interval (0, 0.5) and decreases monotonically in the interval (0.5, 1).

\section{Supplementary Experimental Details}
\label{supplementary_experiment}
\subsection{\ourmethod{}'s algorithm description}
\begin{algorithm}[H] 
	\begin{algorithmic}[0] 
		\STATE The real-time feedback state from the environment is defined as Initial States $s$, with its corresponding Reconstructed State $\hat{s}$ and Reconstruction Error as intrinsic reward $R_{int}$; $\alpha(\hat{s})$ is the measure of understanding of the current state.
		\STATE Initialize state autoencoder parameter $ \theta $ and evaluation network parameter $ \omega $.
		\FOR{$t=0,1,2,\ldots$}
		\STATE 	\textbf{Generative Network}:
		\STATE Receive real-time state $s$ from the environment
		\STATE Input $s$ into the state autoencoder to obtain the output $\hat{s}$ and intrinsic reward $R_{int}$
		\STATE Update the loss function: $\text{Loss}_f(\theta) = \frac{1}{2} \left\| s_t -  g_\theta(s_t) \right\|_2^2$
		
		\STATE \textbf{Evaluation Network}:
		\STATE Receive the output $\hat{s}$ from the state autoencoder as input
		\STATE Output $\alpha(\hat{s})$
		\STATE Update the loss function: $\text{Loss}_g(\omega) =\mathbb{E} \left[ y(s) \cdot \log(f_{\omega}(s)) + (1 - y(s)) \cdot \log(1 - f_{\omega}(s)) \right]$
		
		\STATE \textbf{Adaptive Mechanism}:
		\STATE The environment provides extrinsic reward $R(s, a)$, the state autoencoder provides $R_{\text{int}}(s, a)$, and the evaluation network provides $\alpha(\hat{s})$
		\STATE Redefine the intrinsic reward to balance exploration and exploitation adaptively:
		\[ R_{\text{total}}(s, a) = R(s, a) +  (1-\alpha(\hat{s}))R_{\text{int}}(s, a) \]
		\ENDFOR
	\end{algorithmic}
	\caption{Exploration and Exploitation Adaptive Balance Method (\ourmethod{})}\label{EEadapt-algorithm}
\end{algorithm}


\begin{figure}
	\centering
	\includegraphics[width=0.9\linewidth]{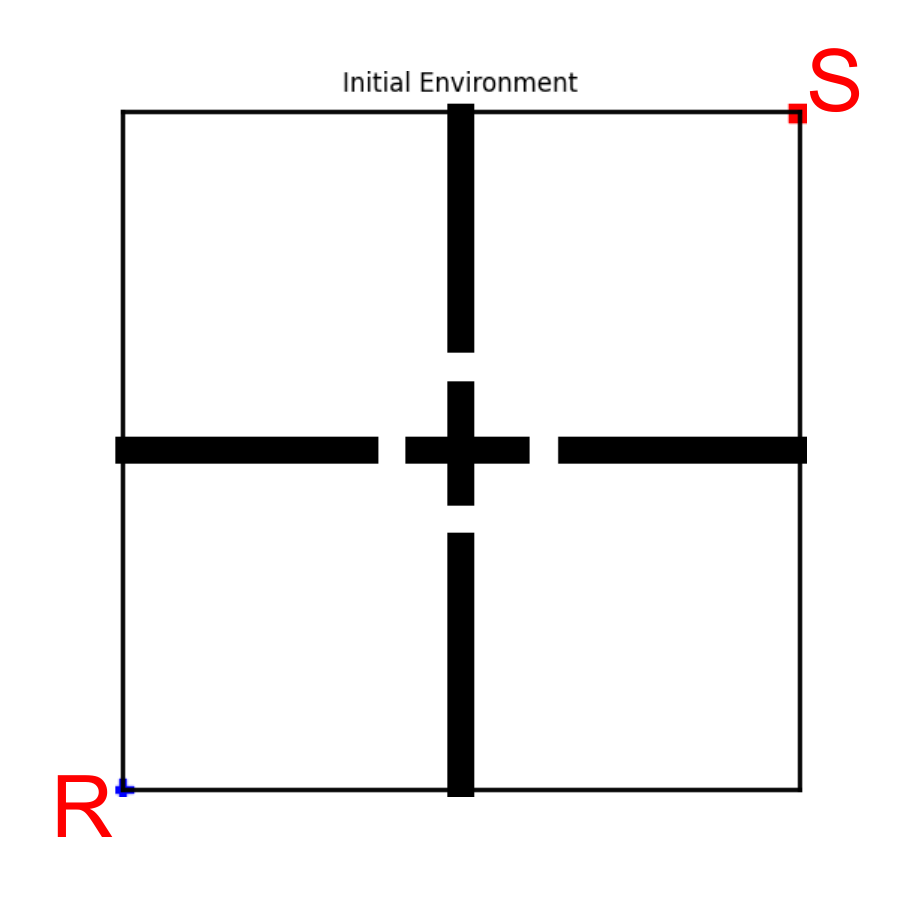}
	\caption{\textbf{Visual Experiments in Dark Chamber} in 2.5 million steps.
	}
	\label{fourrooms0}
\end{figure}

\begin{figure*}
	\centering
	\includegraphics[width=0.9\linewidth]{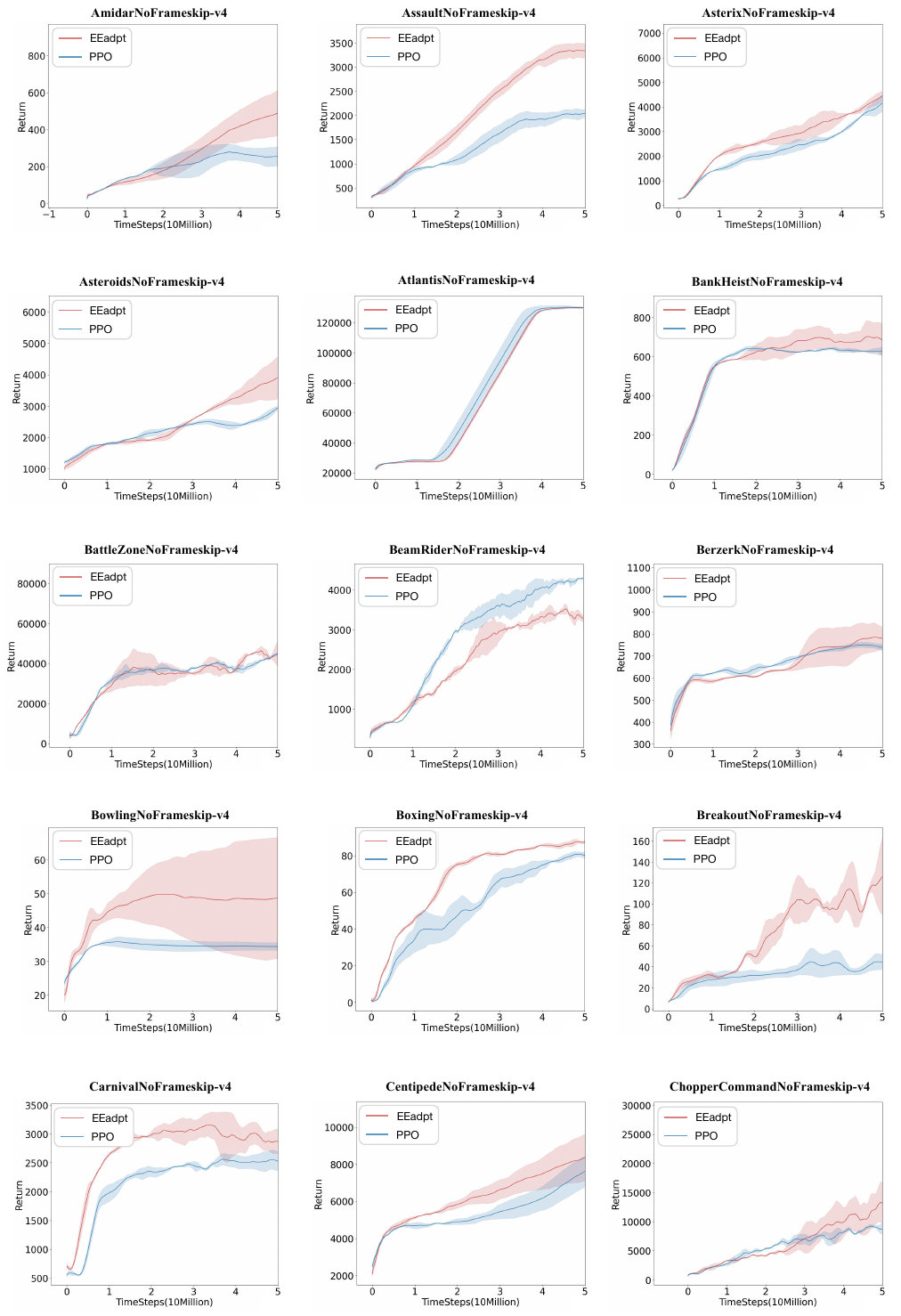}
	\caption{\textbf{Supplementary Expansion Experiments in Atari Part I.} Beyond the Expansion Experiments described in the main text, a comparison of scores across all remaining Atari environments. 
	}
	\label{atari}
\end{figure*}

\begin{figure*}
	\centering
	\includegraphics[width=0.9\linewidth]{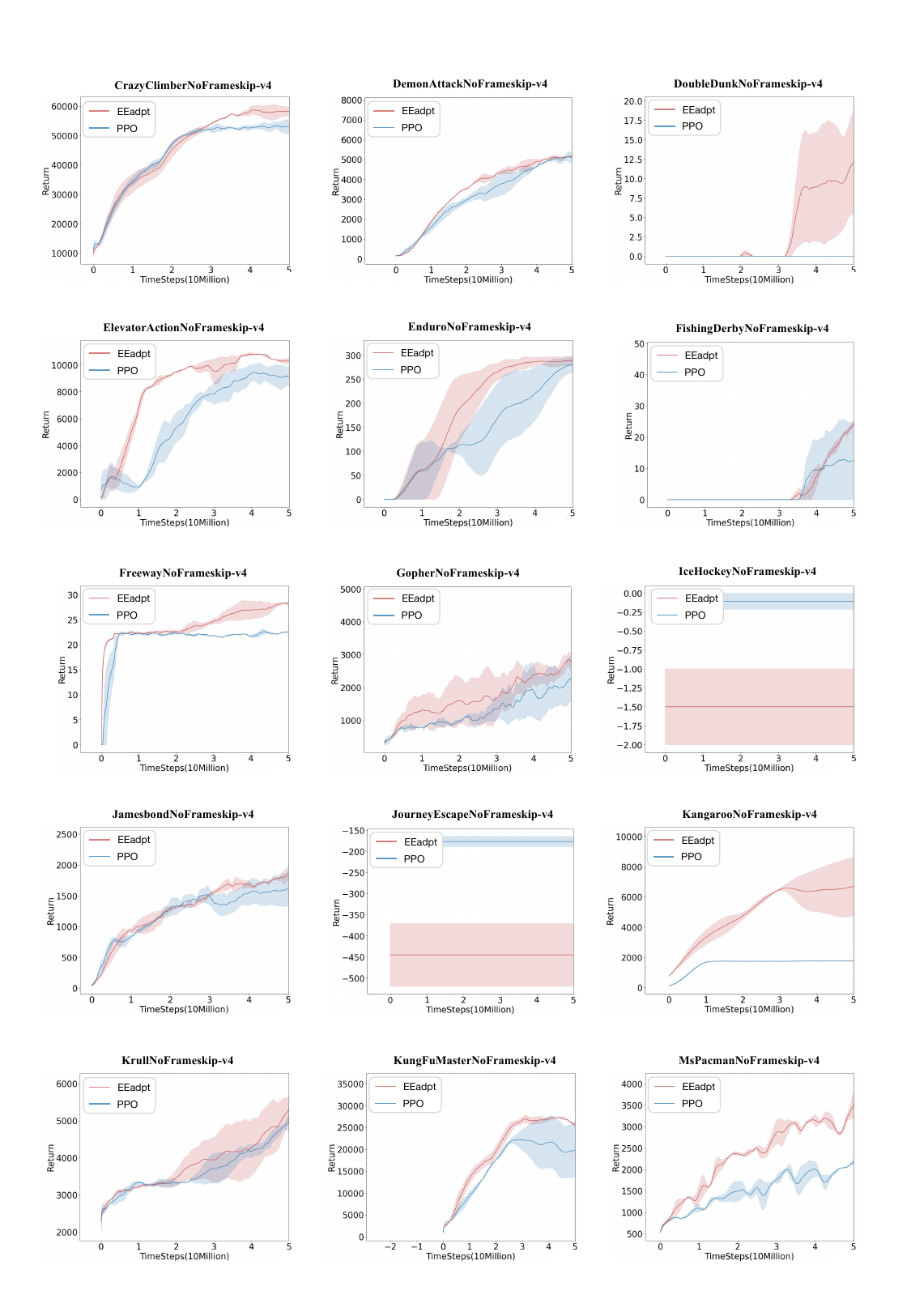}
	\caption{\textbf{Supplementary Expansion Experiments in Atari Part II.} Beyond the Expansion Experiments described in the main text, a comparison of scores across all remaining Atari environments. 
	}
	\label{atari}
\end{figure*}

\begin{figure*}
	\centering
	\includegraphics[width=0.9\linewidth]{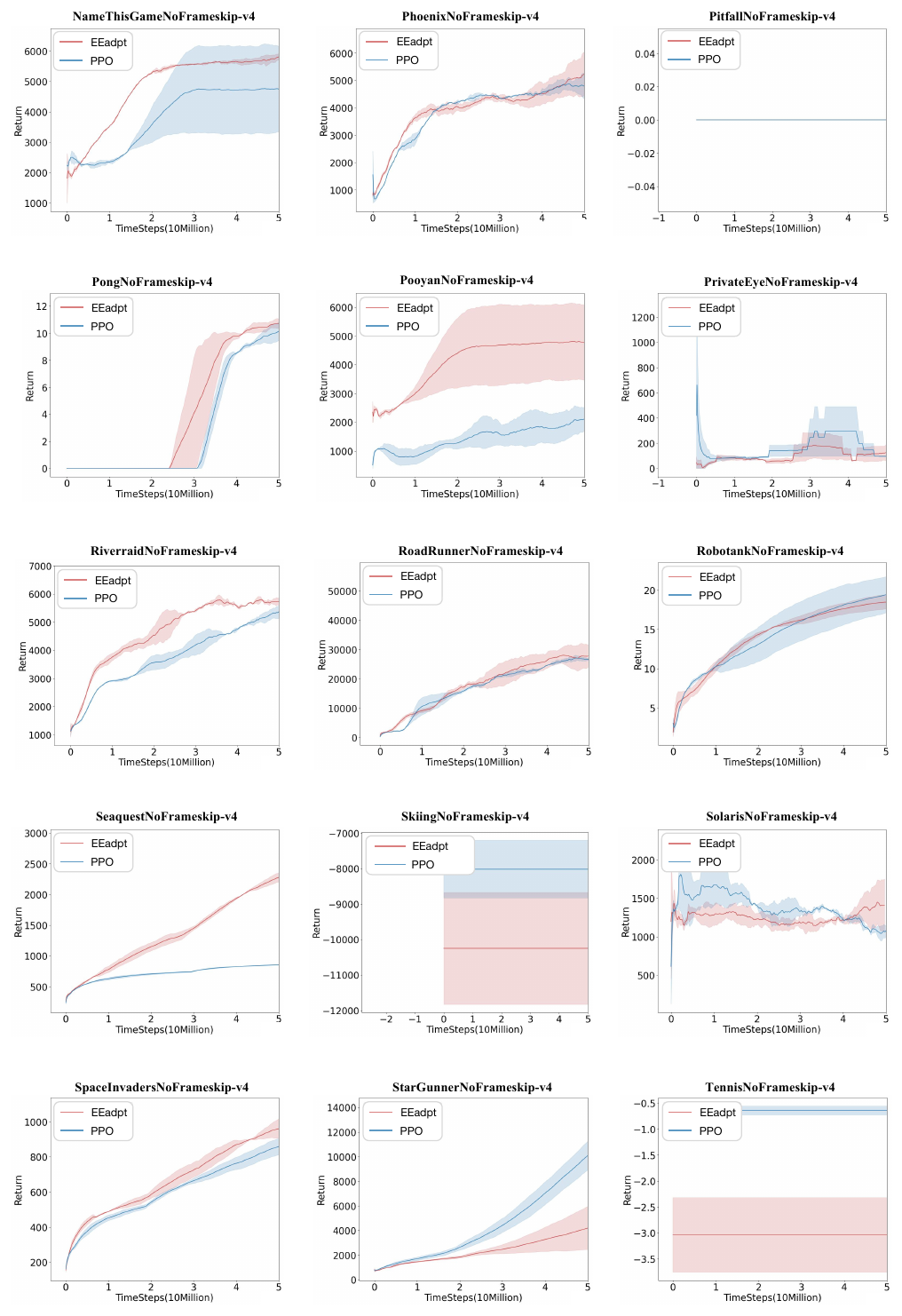}
	\caption{\textbf{Supplementary Expansion Experiments in Atari Part III.} Beyond the Expansion Experiments described in the main text, a comparison of scores across all remaining Atari environments. 
	}
	\label{atari}
\end{figure*}

\begin{figure*}
	\centering
	\includegraphics[width=0.9\linewidth]{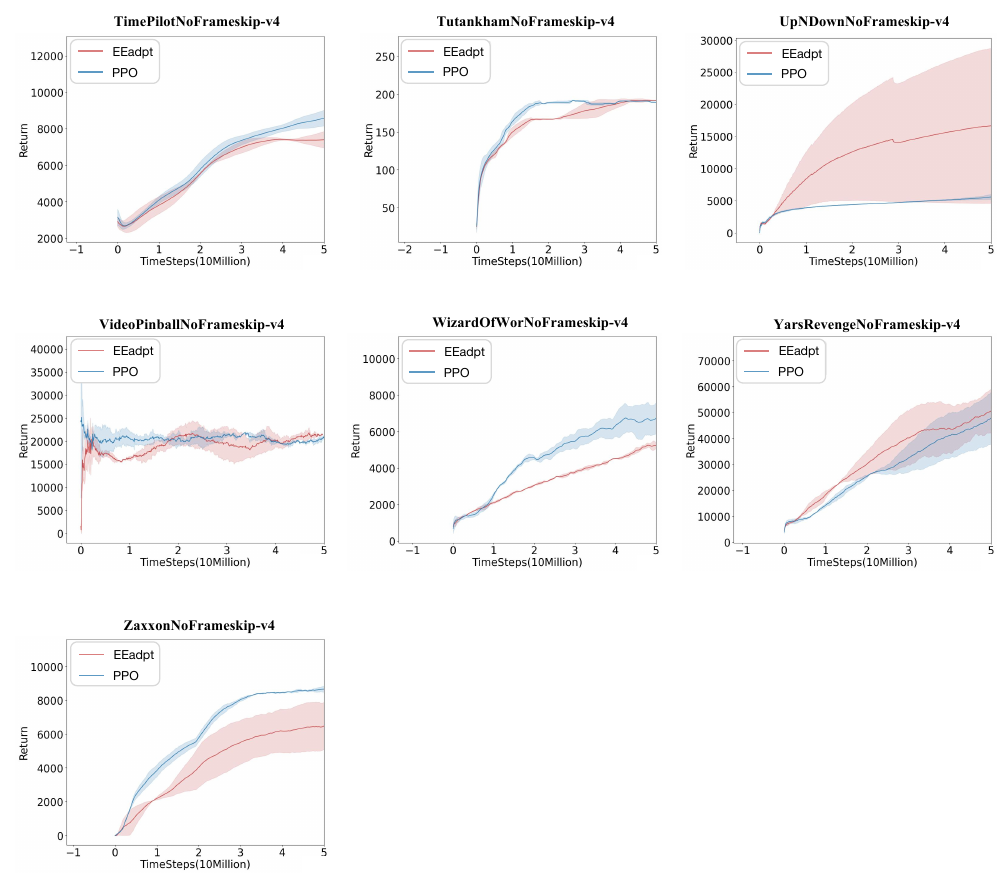}
	\caption{\textbf{Supplementary Expansion Experiments in Atari Part IV.} Beyond the Expansion Experiments described in the main text, a comparison of scores across all remaining Atari environments. 
	}
	\label{atari}
\end{figure*}

\begin{figure*}
	\centering
	\includegraphics[width=0.99\linewidth]{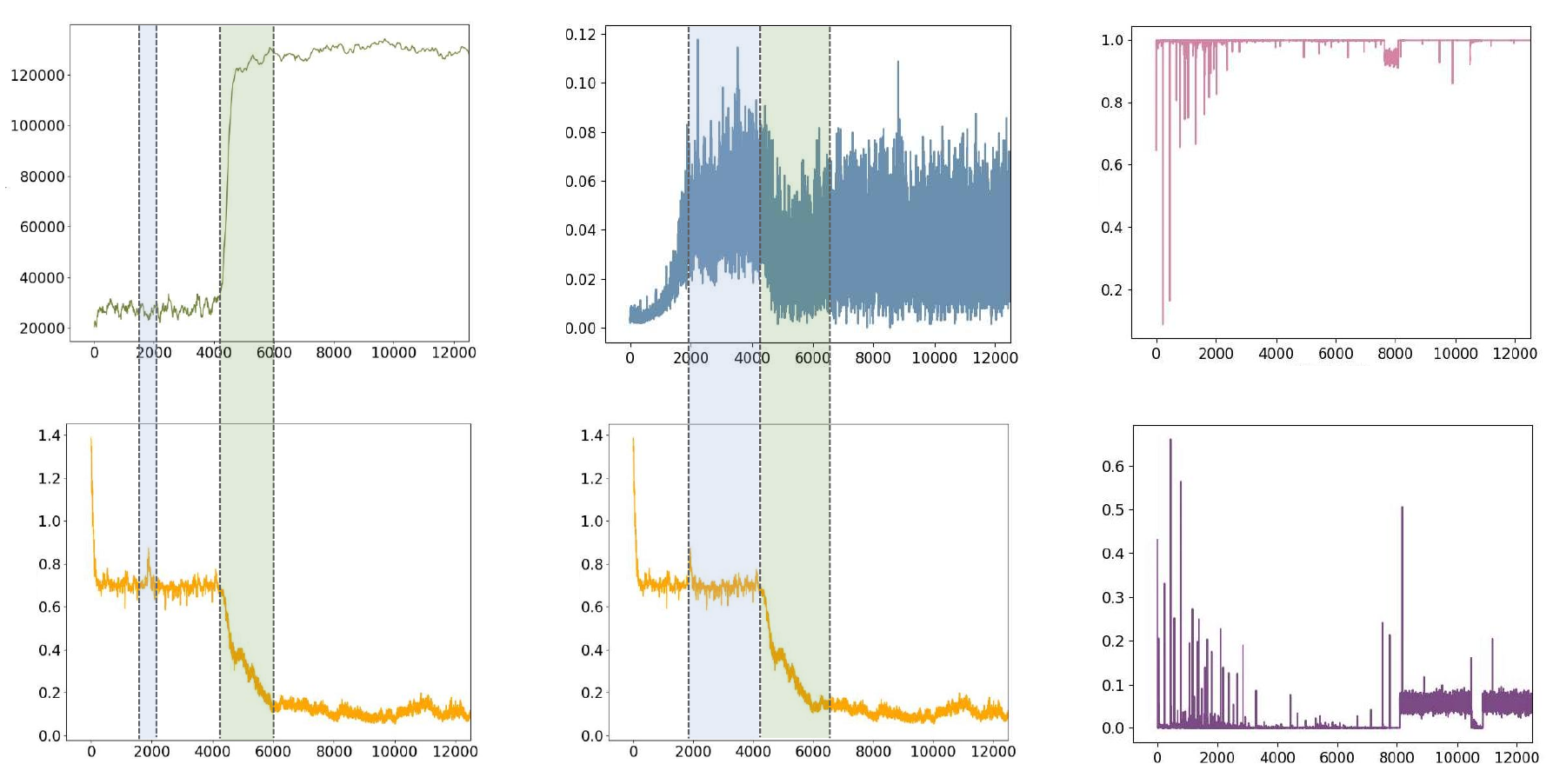}
	\caption{Entropy Variation on Policy Optimization. To more clearly illustrate the subtle changes in policy optimization, the data are presented without averaging multiple seeds or applying any smoothing operations.}
	\label{alignment}
\end{figure*}

\end{document}